\documentclass[conference]{IEEEtran}
\IEEEoverridecommandlockouts
\usepackage{amsthm} 
\usepackage{times}  
\usepackage{helvet} 
\usepackage{courier}
\usepackage{url}  
\usepackage{graphicx}
\usepackage{mathptmx}
\usepackage{latexsym} 
\usepackage{amsmath}
\usepackage[noend]{algpseudocode}
\usepackage{fancyvrb}
\usepackage{xspace}
\usepackage[table]{xcolor}   
\usepackage{listings} 
\usepackage{graphicx}
\usepackage{amssymb}
\usepackage{bbding}
\usepackage{caption}
\usepackage{subcaption}
\usepackage{dsfont}
\usepackage{todonotes}
\usepackage{algorithm2e}
\usepackage{hyperref}

\usepackage[para]{footmisc}%

\theoremstyle{definition}
\newtheorem{definition}{Definition}
\newtheorem{theorem}{Theorem}

\newcommand{\sketchedeq}{\ensuremath{X~?{=}~Y}\xspace}
\newcommand{\sketchedplus}{\ensuremath{X~?{+}~Y}\xspace}
\newcommand{\reified}{\textit{reified}\xspace}

\newcommand{\qone}{\ensuremath{\textbf{Q}_1}\xspace}
\newcommand{\qtwo}{\ensuremath{\textbf{Q}_2}\xspace}
\newcommand{\qthree}{\ensuremath{\textbf{Q}_3}\xspace}
\newcommand{\qfour}{\ensuremath{\textbf{Q}_4}\xspace}
\newcommand{\qfive}{\ensuremath{\textbf{Q}_5}\xspace}
\newcommand{\eplus}{\ensuremath{\mathbf{E}^+}\xspace}
\newcommand{\eminus}{\ensuremath{\mathbf{E}^-}\xspace}
\newcommand{\dist}{\ensuremath{\textit{dist}}\xspace}

\newcommand{\metae}{\textit{metaE(\ensuremath{P,S,\eplus,\eminus})}\xspace}
\newcommand{\metad}{\textit{metaD(\ensuremath{S,D})}\xspace}
\newcommand{\metar}{\textit{metaR(\ensuremath{S,D})}\xspace}
\newcommand{\metac}{\textit{metaC(\ensuremath{P})}\xspace}

\definecolor{awesome}{rgb}{1.0, 0.13, 0.32}
\newcommand{\sergey}[1]{#1}

\newcommand{\addedbegin}{}
\newcommand{\addedend}{}

\def\BibTeX{{\rm B\kern-.05em{\sc i\kern-.025em b}\kern-.08em
    T\kern-.1667em\lower.7ex\hbox{E}\kern-.125emX}}
\begin{document}

\title{Sketched Answer Set Programming\\
\thanks{This work has been partially funded by the ERC AdG SYNTH (Synthesising inductive data models) }
}

\author{\IEEEauthorblockN{Sergey Paramonov}
\IEEEauthorblockA{\textit{KU Leuven}\\
Leuven, Belgium \\
sergey.paramonov@kuleuven.be}
\and
\IEEEauthorblockN{Christian Bessiere}
\IEEEauthorblockA{\textit{LIRMM} \textit{CNRS}\\
Montpellier, France\\
bessiere@lirmm.fr}
\and
\IEEEauthorblockN{Anton Dries}
\IEEEauthorblockA{\textit{KU Leuven}\\
Leuven, Belgium \\
anton.dries@kuleuven.be}
\and
\IEEEauthorblockN{Luc De Raedt}
\IEEEauthorblockA{\textit{KU Leuven}\\
Leuven, Belgium\\
luc.deraedt@kuleuven.be}
}

\maketitle

\begin{abstract}
Answer Set Programming (ASP) is a powerful modeling formalism for combinatorial problems. However, writing ASP models can be hard. We propose a novel method, called Sketched Answer Set Programming (SkASP), aimed at facilitating this. In SkASP, the user writes partial ASP programs, in which uncertain parts are left open and marked with question marks. In addition, the user provides a number of positive and negative examples of the desired program behaviour.  SkASP then synthesises a complete ASP program. This is realized by  rewriting the SkASP program into another ASP program, which can then be solved by traditional ASP solvers.  We evaluate our approach on 21 well known puzzles and combinatorial problems inspired by Karp’s 21 NP-complete problems and on publicly available ASP encodings.
\end{abstract}

\section{Introduction}\label{sec:intro}
Many AI problems can be formulated as constraint satisfaction problems that can
be solved by state-of-the-art constraint programming (CP)  \cite{handbookcp} or answer set programming (ASP) techniques \cite{whatisasp}. 
Although these frameworks provide declarative representations that are in principle
easy to understand, writing models in such languages is not always easy.


On the other hand, for traditional programming languages, there has been significant attention for techniques that are able
to complete \cite{sketching_phd_thesis} or learn a program from examples \cite{gulwani2015inductive}. 
The idea of program sketching is to start from a sketched program and some examples
to complete the program. A sketched program is essentially a program where some of the tests and constructs
are left open because the programmer might not know what exact instruction to use.  For instance, 
when comparing two variables $X$ and $Y$, the programmer might not know whether to use $X < Y$ or $X \leq Y$ or $X > Y$ and write 
$X~{?}{=}~Y$ instead (while also specifying the domain of ${?}{=}$, that is, which concrete operators are allowed). 
By providing a few examples of desired program behaviour and a sketch, the target program can then be automatically found.
Sketching is thus a form of "lazy" programming as one does not have to fill out all details in the programs; 
it can also be considered as program synthesis although there are strong syntactic restrictions on 
the programs that can be derived; and it can be useful for repairing programs once a bug in a program has been detected.
Sketching has been used successfully in a number of applications \cite{sketching_original,sketch_recent,jsketch} to synthesise imperative programs.
It is these capabilities that this paper brings to the field of ASP.

As a motivating example assume 
one needs to solve 
the Peacefully Coexisting Armies of Queens, a version of the $n$-queens problem with black and white queens, where queens of the same color do not attack each other. 
One might come up with the following sketched program (where 
$R_w$ ($C_b$) stand for the variable representing the row (column) of a white (black) queen): 
\begin{lstlisting}[caption=Peacefully Coexisting Armies of Queens,label=lst:queens,basicstyle=\scriptsize\ttfamily,escapechar=@,numbers=left,xleftmargin=11pt]
:- queen(w,Rw,Cw), queen(b,Rb,Cb), Rw ?= Rb. 
:- queen(w,Rw,Cw), queen(b,Rb,Cb), Cw ?= Cb.
:- queen(w,Rw,Cw), queen(b,Rb,Cb), Rw ?+ Rb ?= Cw ?+ Cb. @\label{lst:queens:line:arithmetic}@
\end{lstlisting}
This program might have been inspired by a solution written in the constraint programming language Essence available from the CSP library \cite{csplib:prob110}.  
Intuitively, the sketched ASP specifies constraints on the relationship between two queens on the rows (first rule), columns (second rule) and diagonals (third rule), 
but it expresses also uncertainty about the particular operators that should be used between the variables 
through the built-in alternatives for ${?}{=}$  (which can be instantiated to one of $=,\neq,<,>,\leq,\geq$) and for ${?}+$ (for arithmetic operations). 
When providing an adequate set of examples to the ASP, the SkASP solver will then produce the correct program.

The key contributions of this paper are the following: 1) we adapt the notion of sketching for use with Answer Set Programming;
2) we develop an approach (using ASP itself) for computing solutions to a sketched Answer Set Program;
3) we contribute some simple complexity results on sketched ASP; 
and 4) we investigate the effectiveness and limitations of sketched ASP on a dataset of 21 typical ASP programs.

\newcommand{\rulesep}{\unskip\ \vrule\ }
\newcommand{\hrulesep}{\unskip\ \hrule\ }

\begin{figure*}[htb]
\begin{subfigure}[t]{0.315\textwidth}
  \renewcommand{\figurename}{Listing}
  \vspace{3pt}
  \begin{Verbatim}[fontsize=\scriptsize,numbers=left,xleftmargin=0mm,commandchars=\\\{\}]
[SKETCH]
reached(Y) :- cycle(a,Y).\label{line:reached}
reached(Y) :- cycle(X,Y), reached(X).\label{line:reached2}
    :- ?p(Y), ?not ?q(Y). \label{line:ham_constraint}
[EXAMPLES]
positive: cycle(a,b). cycle(b,c). ...
negative: cycle(a,b). cycle(b,a).
[SKETCHEDVAR]
?p/1 : node, reached
?q/1 : node, reached
[FACTS]
node(a). node(b). node(c).
[EXAMPLES]
positive: cycle(a,b). cycle(b,c)...  \label{line:example}
negative: cycle(a,b). cycle(b,a).
\end{Verbatim}
\caption{Hamiltonian Cycle (ASP due to \protect\cite{ASPbook})} \label{lst:ham}
\end{subfigure}
\rulesep
\begin{subfigure}[t]{0.66\textwidth}
  \vspace{3pt}
  \renewcommand{\figurename}{Sketch}
  \begin{Verbatim}[fontsize=\scriptsize,numbers=left,xleftmargin=5.7mm,commandchars=\\\{\}]
%%%%% EXAMPLES AND DECISIONS %%%%%%
positive(0). cycle(0,a,b). cycle(0,b,c). cycle(0,c,a).\label{line:example_rewritten}
reified_q_choice(c_node). reified_q_choice(c_reached).
1 \{decision_q(X) : reified_q_choice(X)\} 1. \label{line:decision1}
1 \{decision_not(pos) ; decision_not(neg)\} 1.\label{line:decision2}
%%%%% INFERENCE RULES %%%%%%
reified_q(E,c_node,X0) :- node(X0),examples(E). \label{line:reified_q}
reified_q(E,c_reached,X0) :- reached(E,X0).
reached(E,Y) :- cycle(E,a,Y), examples(E). \label{line:reach_rewritten}
reached(E,Y) :- cycle(E,X,Y), reached(E,X), examples(E).
reified_not(E,pos,Q,Y) :- reified_q(E,Q,Y). \label{line:reified_not1}
reified_not(E,neg,Q,Y) :- not reified_q(E,Q,Y), dom1(Y),dom2(Q),examples(E). \label{line:reified_not2}
%%%%% POSITIVE/NEGATIVE SKETCHED RULES %%%%%%
:- reified_p(E,P,Y), reified_not(E,Not_D,Q,Y), decision_p(P), decision_q(Q),\label{line:positive_rewritten2}
               decision_not(Not_D), positive(E).\label{line:positive_rewritten3}
neg_sat(E) :- reified_p(E,P,Y),  \label{line:negative_rewritten1} reified_not(E,Not_D,Q,Y), decision_p(P), \label{line:negative_rewritten2}
decision_q(Q), decision_not(Not_D), negative(E).\label{line:negative_rewritten3}
:- not neg_sat(E), negative(E).
\end{Verbatim}
\caption{ASP core of rewritten Fig. \protect \ref{lst:ham} (only predicate $q$ shown, domains, facts, etc omitted)} \label{lst:rewriting}
\end{subfigure}
\hrulesep
\begin{subfigure}[t]{0.43\textwidth}
  \renewcommand{\figurename}{Sketch}
  \vspace{2pt}
\begin{Verbatim}[fontsize=\scriptsize,numbers=left,xleftmargin=0mm,commandchars=\\\{\}]
[SKETCH] %constraints on squares, rows, columns
:- cell(X,Y,N), cell(X,Z,M), Y ?= Z, N ?= M. \label{lst:sketch_sudoku:ineq}
:- cell(X,Y,N), cell(Z,Y,M), X ?= Z, N ?= M.
insquare(S,N) :- cell(X,Y,N), square(S,X,Y).
:- num(N), squares(S), ?not insquare(S,N). 
\end{Verbatim}
\caption{Sudoku (core). \protect ASP code from \cite{asp_tutorial_sudoku}} \label{lst:sudoku}
\end{subfigure} 
\rulesep
\begin{subfigure}[t]{0.55\textwidth}
  \vspace{2pt}
  \renewcommand{\figurename}{Listing}
\begin{Verbatim}[fontsize=\scriptsize]
[SKETCH] %constraints on rows and columns
:- cell(X,Y,N), cell(X,Z,M), Y ?= Z, N ?= M.
:- cell(X,Y,N), cell(Z,Y,N), X ?= X, N ?= M.
[EXAMPLES]
positive: cell(1,1,a). cell(1,2,b). cell(1,3,c)...
\end{Verbatim}
\caption{Latin Square (based on Sudoku; core)} \label{lst:latin_square}
\end{subfigure}
\caption{Collection of sketches and an example of rewriting used in the paper}
\label{fig:table_with_sketches}
\end{figure*}

\section{ASP and Sketching }\label{sec:formal}
Answer Set Programming (ASP) is a form of declarative programming based on the stable model semantics \cite{stable_models} of logic programming \cite{whatisasp}. We follow the standard syntax and semantics of ASP as described in the Potassco project \cite{ASPbook}. A \textit{program} is a set of rules of the form
$
  a \leftarrow a_1, \dots a_k, \textit{not } a_{k+1}, \dots, \textit{not } a_{n}
$
A positive or negative atom is called a \textit{literal}, $a$ is a positive propositional literal, called a \textit{head}, and for $i$ between $1$ and $k$, $a_i$ is a positive propositional atom; and for $i$ between $k+1$ and $n$, $\textit{not }a_i$ is a negative propositional literal. The \textit{body} is the conjunction of the literals. A rule of the form $a \leftarrow .$ is called a \textit{fact} and abbreviated as $a.$ and a rule without a head specified is called an \textit{integrity constraint} ($a$ is $\bot$ in this case). \textit{Conditional literals}, written as $a : l_1, \dots, l_n$, and \textit{cardinality constraints}, written as $c_{\min} \{ l_1,\dots,l_n \} c_{\max}$, are used ($l_1, \dots,l_n$ are literals here, and $c_{\min}, c_{\max}$ are non-negative integers). 
A conditional atom holds if its condition is satisfied and a cardinality constraint is satisfied if between $c_{\min}$ and $c_{\max}$ literals hold in it. 
Furthermore, as ASP is based on logic programming and also allows for {\em variables}, denoted
in upper-case, the semantics of a rule or expression with a variable is the same as that of its set of ground instances. 
We restrict the ASP language to the NP-complete subset specified here. For more details on ASP, see \cite{ASPbook,eiter_asp_primer}.

We extend the syntax of ASP with \textit{sketched} language constructions. 
Instead of allowing only atoms of the form $p(t_1, ...,t_n)$, where $p/n$ is a predicate
and the $t_i$ are terms (variables or constants), we now allow to use
{\em sketched atoms} of the form $?q(t_1, ...,t_n)$ where $?q$ is a {\em sketched predicate variable}
with an associated domain $d_q$ containing actual predicates of arity $n$. 
The meaning of the sketched atom $?q(t_1, ... ,t_n)$ is that it can be replaced
by any real atom $p(t_1, ...,t_n)$ provided that $p/n \in d_q$.
It reflects the fact that the programmer does not know which $p/n$ from $ d_q$ should be used. 
Sketched atoms can be used in the same places as any other atom. 

We also provide some syntactic sugar for some special cases and variants,
in particular, we use a {\em sketched inequality} \sketchedeq, a  {\em sketched} arithmetic operator \sketchedplus (strictly speaking, this is not a sketched predicate but an operator, but we only make this distinction where needed), and {\em sketched negation} ${?}\textit{not } p(X)$ (which is, in fact, a sketched operator of the form ''?not <atom>``; it always has as input a positive atom and its domain is \{atom, -atom\}, where -atom is a syntactically new atom, which represents the negation of the original atom).  The domain of \sketchedeq is the set $\{=,\neq,<,>,\geq,\leq,\top\}$, where $\top$ is the atom that is always satisfied by its arguments, the domain of \sketchedplus is the set $\{ +, -, \times, \div, \dist \}$ where $\dist(a,b)$ is defined as $| a - b |$, and the domain of ${?}\textit{not}$  is $\{\emptyset, not\}$. An example of sketched inequality can be seen in Line \ref{lst:sketch_sudoku:ineq} of Figure \ref{lst:sudoku}, examples of sketched predicates and negation in Line \ref{line:ham_constraint} of Figure \ref{lst:ham} and sketched arithmetic in Line \ref{lst:queens:line:arithmetic} of Sketch \ref{lst:queens}.

\textit{A sketched variable} is a sketched predicate, a sketched negation, a sketched inequality or a sketched arithmetic operator. The set of all sketched variables is referred to  as $S$. 
\sergey{Predicate $p$ \textit{directly positively (negatively) depends} on $q$ iff $q$ occurs positively (negatively) in the body of a rule with $p$ in the head or $p$ is a sketched predicate and $q$ is in its domain; $p$ {\em depends (negatively) on} $q$ iff $(p,q)$ is in the transitive closure of the direct dependency relation. A sketch is \textit{stratified} iff there is no negative cyclic dependency. We restrict programs to the stratified case.}
An \textit{example} is a set of ground atoms. 

A \textit{preference} is a function from $\Theta$ (possible substitutions) to $\mathbb{Z}$. A substitution $\theta$ is \textit{preferred} over $\theta'$ given preferences $f$ iff for all $s_i \mapsto d_i \in \theta$ and $s_{i} \mapsto d_{i}' \in \theta'$ it holds that $f(s_i \mapsto d_i) \geq f(s_i \mapsto d_{i}')$ and at least one inequality is strict. First, when $f(\theta)$ is constant,
all substitutions are equal and there are no preferences (all equally preferred). Because specifying preferences might impose an extra burden on the user, we also provide default preferences for the built-in sketched variables (like inequality, etc), cf. the experimental section.

\addedbegin
\textit{The Language of Sketched Answer Set Programming (SkASP)}
supports some of the language features of ASP. 
The language of SkASP has the following characteristics:

\begin{itemize}
\item it allows for a set of rules of the form $a \leftarrow b_1, \dots, b_n, \textit{not }c_1, \dots, \textit{not }c_m.$;
\item predicates (such as a predicate $p/n$ or comparison $\leq$) and operators (such as arithmethic $+,-,\times$, etc) in these rules can be sketched; 
\item aggregates can be used in the body of the rules as well (stratified; see Extension Section \ref{section:aggregates});
\item the SkASP program has to be stratified; 
\item the choice rules are not allowed.
\end{itemize}

The key idea behind our method is that the SkASP program is rewritten into a normal ASP program (with choice rules, etc.) in order to obtain a solution through the use of an ASP solver.
As we will see in Theorem \ref{theorem:complexity_sat}: the language of SkASP stays within the complexity bounds of normal ASP, which makes the rewriting possible (SkASP$\mapsto$ASP).

Let us now formally introduce the problem of SkASP.
\addedend
\begin{definition}[The Problem of Sketched Answer Set Programming]
  Given  a sketched answer set program $P$ with sketched variables $S$ of domain $D$ and preferences $f$, and positive and negative sets of examples \eplus and \eminus, the \emph{Sketched Answer Set Problem} is to find all substitutions $\theta:S\mapsto D$ preferred by $f$ such that $P\theta \cup \{e\}$ has an answer for all $e$ in \eplus and for no $e$ in \eminus. 
The decision version 
of SkASP  asks whether there exists such a substitution $\theta$.
\end{definition}

\section{Rewriting Schema}\label{sec:method}
One might consider a baseline approach that would enumerate all instances of the ASP sketch, and in this way produce
one ASP program for each assignment that could then be tested on the examples. 
This naive grounding and testing approach is, however, infeasible: the number of possible combinations grows exponentially with the number of sketched variables. E.g., for the sketch of the Radio Frequency Problem \cite{fap} there are around $10^5$ possible assignments to the sketched variables. Multiplied by the number of examples, around a million ASP programs would have to be generated and tested. This is  infeasible in practice.

The key idea behind our approach is to rewrite a SkASP problem $(P,S,D,f,\eplus,\eminus)$ into an ASP program such that the original sketching program has a solution iff the ASP program has an answer set. 
This is achieved by 1) inserting \textit{decision variables} into the sketched predicates,
and 2) introducing {\em example identifiers} in the predicates. 

The original SkASP problem is then turned into an ASP problem  
on these decision variables and solutions to the ASP problem allow to reconstruct the SkASP substitution.

The rewriting procedure has four major steps: \textit{example expansion}, \textit{substitution generation}, \textit{predicate reification} and \textit{constraint splitting}. \sergey{(Here we follow the notation on meta-ASP already used in the literature \cite{inductive_asp,asp_meta}.)} 

\vspace{3pt}
\textbf{Example Identifiers} 
To allow the use of multiple examples in the program, every relevant predicate is extended with
an extra argument that represents the example identifier.  The following steps are used
to accommodate this in the program,\sergey{ denoted as \metae.}
\begin{enumerate}
\item
  Let $\textit{SP}$ be the set of all predicates that depend on a predicate occurring in one of the examples. 
\item
  Replace each literal $p(t_1, ...,t_n)$ for a predicate $p \in \textit{SP}$ in the program $P$ by 
the literal $p(E,t_1, ...,t_n)$, where $E$ is a variable not occurring in the program. 
\item
    Add the guard $\textit{examples}(E)$ (the index of all pos./neg. examples) to the body of each rule in $P$.
\item
For each atom $p(t_1,...,t_n)$ in the $i$-th example, add the fact $p(i,t_1,...,t_n)$ to $P$.
\item
  For each positive example $i$, add the fact $\textit{positive}(i)$ to $P$, and for each negative one, the fact $\textit{negative}(i)$.
\end{enumerate}
E.g., the rule in Line \ref{line:reached} of Figure \ref{lst:ham} becomes Line \ref{line:reach_rewritten} of Figure \ref{lst:rewriting}, and the example in Line \ref{line:example} is rewritten as in Line \ref{line:example_rewritten}.


\vspace{3pt}
\textbf{Substitution Generation} 
We now introduce the decision variables, \sergey{referred as \metad}:
\begin{enumerate}
\item
For each sketched variable $s_i$ with domain $D_i$ 
$$1~\{ \textit{decision}\_s_i(X) : d_i(X) \}~1 .$$ 
\item 
For each value $v$ in $D_i$, add the fact $d_i(v)$.
\end{enumerate}
This constraint ensures that each answer set has exactly one value from the domain assigned
to each sketched variable. 
This results in a one-to-one mapping between decision atoms and sketching substitution $\theta$. An example can be seen in Lines \ref{line:decision1} and \ref{line:decision2} of Figure \ref{lst:rewriting}.


\vspace{3pt}
\textbf{Predicate Reification} 
We now introduce the reified predicates, \sergey{referred as \metar}
\begin{enumerate}
\item
Replace each occurrence of a sketched atom $s(t_1,...,t_n)$ in a rule of $P$ with the atom $\reified\_s( D, t_1, \dots, t_n)$, 
and add $\textit{decision}\_s(D)$ to the body of the rule.
\item
For each sketched variable $s$ and value $d_i$ in its domain, add the following rule to $P$: 
$$\reified\_s(d_i,X_1, \dots, X_n) \leftarrow d_i(X_1, \dots, X_n).$$
where the first argument is the decision variable for $s$.
\end{enumerate}
Thus, semantically $\reified\_s(d_i,X_1, \dots, X_n)$ is equivalent to $d_i(X_1, \dots, X_n)$  
and $\textit{decision}\_s(d_i)$ indicates that the predicate $d_i$ has been selected for the sketched variable $s$.
Notice that we focused here on the general case of a sketched predicate ${?}p(\dots)$. 
It is straightforward to adapt this for 
the sketched inequality, negation and arithmetic. Examples of reification can be seen in Lines \ref{line:reified_q} of Figure \ref{lst:rewriting} for the sketched ${?}q$ of the sketch in Figure \ref{lst:ham} and in Lines \ref{line:reified_not1}, \ref{line:reified_not2} for reified negation.

\vspace{3pt}
\textbf{Integrity Constraint Splitting} 
\sergey{(referred as \metac)}
\begin{enumerate}
  \item Replace each integrity constraint $\leftarrow \textit{body}$ by: 
$$\leftarrow \textit{body}, \textit{positive}(E) $$
$$\textit{negsat}(E) \leftarrow \textit{body}, \textit{negative}(E) $$
\item
And add the rule to the program:
$$ \leftarrow \textit{negative}(E), \textit{not} ~ \textit{negsat}(E).$$ 
\end{enumerate}
This will ensure that all positives and none of the negatives have a solution. For example, the constraint in Line \ref{line:ham_constraint} of Figure \ref{lst:ham} is rewritten into a positive constraint in Lines \ref{line:positive_rewritten2},\ref{line:positive_rewritten3} and into a negative in Lines \ref{line:negative_rewritten1}, \ref{line:negative_rewritten2}, \ref{line:negative_rewritten3}.

Another important result is that the preferences do not affect the decision complexity. Proofs can be found in the supplementary materials. 
\begin{theorem}[Sound and Complete Sketched Rewriting]
  A sketched ASP program $(P, S, D, f, \eplus,\eminus)$ has a satisfying substitution $\theta$ iff the meta program \\$T=$ \mbox{$\metae\cup\metad \cup \metar \cup \metac$} has an answer set.
  \label{theorem:rewriting}
\end{theorem}
Interestingly, the sketched ASP problem is in the same complexity class as the original ASP program.
\begin{theorem}[Complexity of Sketched Answer Set Problem]
  The decision version of propositional SkASP is in NP.
  \label{theorem:complexity_sat}
\end{theorem}
\begin{proof} 
Follows from the encoding of SkASP into a fragment of ASP which is in NP. 
\end{proof}
\textbf{Dealing with preferences} 
Preferences are, as we shall show in our experiments, useful to restrict the number of solutions.
We have implemented preferences using a post-processing approach (which will also allow to apply the schema to other formalisms such as CP or IDP \cite{idp}).
We first generate the set of all solutions $O$ (without taking into account the preferences), and then post-process $O$. Basically, we filter out
from $O$ any solution that is not preferred (using tests on pairs $(o,o')$ from $O \times O$). 
The preferences introduce a partial order on the solutions. 
\sergey{\addedbegin For example, \addedend assume ${?}p$ (${?}q$) can take value  $p_1$ ($q_1$) with preference of $1$ and $p_2$ ($q_2$) with $2$. If $(p_1,q_2)$ and $(p_2,q_1)$ are the only solutions, they are kept because they are incomparable -- $(1,2)$ is not dominated by $(2,1)$ (and vice versa). If $(p_1,q_1)$ is also solution, 
$(p_1,q_2)$ and $(p_2,q_1)$ are removed because they are dominated by $(p_1,q_1)$. }

\addedbegin
While the number of potential Answer Sets is in general exponential for
a sketched ASP, the number of programs actually satisfying the
examples is typically rather small (in our experiments, below 10000-20000). If that is not the case, then the problem is 
under-constrained and it needs more examples. No user would be able to
go over a million of proposed programs. 


\addedend



\section{System Extension: Aggregates and Use-Case}\label{section:aggregates}
An aggregate \textit{\#agg} is a function from a set of tuples to an integer. For example,$\#\textit{count}\{ Column,Row: \textit{queen}(Column,Row) \}$ counts the number of instances \textit{queen(Column,Row)} at the tuple level. Aggregates are often useful for modeling. 
However, adding aggregates to non-disjunctive ASP raises the complexity of an AS existence check, unless aggregate dependencies are stratified \cite{aggregates_complexity}. It is possible to add aggregates into our system under the following restrictions: stratified case, aggregates occur in the body in the form ${N = \#\textit{agg}\{ \dots \}}$, sketched with the keyword \verb|?#| , where \textit{\#agg} can be \textit{max}, \textit{min}, \textit{count} and \textit{sum}. This immediately allows us  to model problems such as Equal Subset Sum (for details, see the repository), where one needs to split a list of values, specified as a binary predicate   \verb|val(ID,Value)| into two subsets, such as \verb|subset1(ID)| (and \verb|subset2(ID)| respectively), such that the sum of both subsets is equal. Essentially, we sketch the constraint of the form: 

{\footnotesize\verb|:- S1 != S2, S1 = ?#{V,X:val(X,V),subset1(X)}...|}

Formally, each aggregate can be seen as an expression of the form:
\begin{equation*}
    \begin{aligned}
    S = \#\textit{agg}\{ Z_1,\dots,Z_n: \textit{cond}(\overbrace{X_1,\dots,X_k}^{\textit{internal}}, \overbrace{Y_1,\dots,Y_h}^{\textit{external}},\overbrace{Z_1,\dots,Z_n)}^{\textit{aggregated}} \},\\ \textit{external}(Y_1,\dots,Y_h)
    \end{aligned}
\end{equation*}
	where $S$ is an integer output, and $Y_1,\dots,Y_h$, shortened as $\bar Y$ ($\bar X$ and $\bar Z$ are the same kind of shortening) are bound to other atoms in the rule, to which we refer as \textit{external($\bar Y$)} ("external" with respect to the condition in the aggregate; it is simply shortening for a conjunction of atoms, which share variables with the condition in the predicate). 

To give an example of $\bar X, \bar Y, \bar Z$ in a simple context: if we were to compute an average salary per department in a company, we might have written a rule of the form:

\begin{Verbatim}[fontsize=\small]
avg_sal(A,D) :- A = #avg{S,N: salaries(N,S,D)}, 
                                 department(D).
\end{Verbatim}
Then, $\bar Z$ consists of the variable \verb|S| and \verb|D| is the external variable (with respect to the condition in the aggregate), i.e., $\bar Y$ and $\bar X$ is composed out of the variable \verb|N|, since it is neither used in the aggregation, nor in the other atoms outside of the aggregate.

A sketched aggregate $?\#$, can be reified similarly to the regular sketched atoms, i.e.:
\begin{equation*}
  \textit{reified}(S,sum,\bar Y) \leftarrow S = \#\textit{sum} \{ \bar Z : \textit{cond}(\bar X, \bar Y, \bar Z) \}, \textit{external}(\bar Y).
\end{equation*}
similarly for other aggregate functions; the same rules, e.g., the example extension, apply to aggregate reification.

\sergey{With aggregates we can sketch a significantly larger class of problems. Consider the problem from the Functional Pearls Collection: ``Finding celebrities problem'' \cite{celebrity_problem}\footnote {\scriptsize ASP code: \url{hakank.org/answer_set_programming/finding_celebrities.lp4}}. Problem statement: ``Given a list of people at a party and for each person the list of people they know at the party, we want to find the celebrities at the party.
A celebrity is a person that everybody at the party knows but that only knows other celebrities. At least one celebrity is present at the party.'' The sketch core looks as follows (names are shortened):}
\begin{Verbatim}[fontsize=\small]
n(N) :- N = ?#{ P : p(P) }.
:- c(C), p(C), n(N), S = ?#{P : k(P,C), p(P)}, 
                                      S < N-1.
:- c(C), p(C), not c(P), k(C,P).
\end{Verbatim} 
The last rule is an integrity constraint verifying that no celebrity, \texttt{c}, knows a person who is not a celebrity. The first line sketches a rule that should find what aggregation metric on the people (unary predicate person, \texttt{p}) should be used in the problem. The sketched rule in the second line makes use of this metric, denoted as \texttt{n}, and says that an aggregation should be performed on the binary "knows" predicate, \texttt{k}, (indicating that two persons know each other); so the outcome of the sketched aggregation on the connection between people should be compared to an overall metric on all people individually.

\section{Experimental Evaluation}\label{sec:experiments} 

\newcommand{\scalefigures}{0.99}

\begin{table}[h] \rowcolors{2}{gray!25}{white}
\centering
\scriptsize
\begin{tabular}{l | c c c c c c}
\textbf{Problem} & \textbf{\# Sketched} &\textbf{\# ?=}& \textbf{\# ?+} & \textbf{\# ?not} &\textbf{\# ?p} & \textbf{\# Rules}\\\hline
Graph Clique            &        3              &   1  &  0    &    0    &   2   &    4    \\
3D Matching             &        3              &   3  &  0    &    0    &   0   &    1    \\
Graph Coloring          &        7              &   4  &  0    &    0    &   3   &    2    \\
Domination Set          &        3              &   0  &  0    &    1    &   2   &    5    \\
Exact Cover             &        7              &   2  &  0    &    1    &   4   &    3    \\
Sudoku                  &        5              &   4  &  0    &    1    &   0   &    4    \\
B\&W Queens             &        5              &   3  &  2    &    0    &   0   &    3    \\
Hitting Set             &        3              &   0  &  0    &    1    &   2   &    2    \\
FAP                     &        3              &   0  &  0    &    1    &   2   &    3    \\
Feedback Arc Set        &        4              &   0  &  0    &    2    &   2   &    3    \\
Latin Square            &        4              &   4  &  0    &    0    &   0   &    2    \\
Edge Domination         &        3              &   0  &  0    &    1    &   2   &    5    \\
FAP                     &        5              &   3  &  2    &    0    &   0   &    3    \\
Set Packing             &        4              &   2  &  0    &    0    &   2   &    1    \\
Clique Cover            &        4              &   3  &  0    &    1    &   0   &    3    \\
Feedback Set            &        5              &   0  &  0    &    5    &   0   &    3    \\
Edge Coloring           &        3              &   3  &  0    &    0    &   0   &    3    \\
Set Splitting           &        5              &   2  &  0    &    1    &   2   &    3    \\
N Queens                &        6              &   4  &  2    &    0    &   0   &    3    \\
Vertex Cover            &        3              &   0  &  0    &    1    &   2   &    4    \\
Subg. Isomorph.         &        5              &   2  &  0    &    1    &   2   &    4    \\
\end{tabular}

\caption{\addedbegin Dataset summary: the number of sketched variables, of rules, of particular types of sketched variables, e.g., ``\# ?not'', indicates how many atoms with the sketched negation are in the program.\addedend}
\label{tab:dataset_description}
\end{table}
For the experimental evaluation we have created a dataset consisting 
of 21 classical combinatorial problems among which most are
NP-complete. \addedbegin For the problem list and precise sketch specifications used in the experiments, we refer to Table \ref{tab:dataset_description}. \addedend
All problems, their code, and implementation details, can be found in the accompanying Github repository: {\url{https://github.com/SergeyParamonov/sketching}}

\textbf{Dataset of Sketches.}
The key challenge in evaluating program synthesis techniques such as SkASP is the absence of 
benchmark datasets (as available in more typical machine learning tasks). At the same time,
although there are many example ASP programs available in blogs, books or come with software,
these typically employ advanced features (such as incremental grounding, optimization or external sources) which are not supported by SkASP as yet.
Therefore we had to design our own dataset in a systematic way (and put it in the public domain).
The dataset is based on a systematic concept (the 21 problems by Karp). When 
we could find encodings for these problem (such as Sudoku in Figure \ref{lst:sudoku} from \cite{asp_tutorial_sudoku} and Hamiltonian Cycle in Figure \ref{lst:ham} from \cite{ASPbook}) we took these problems, in all other cases we developed a solution according to the standard generate and test development methodology of ASP.
Specifically (see \qfive) we looked for different encodings in the public domain of ASP’s favorite -- the N-queens problem (these encoding can tackle even its NP-complete version \cite{complexity_nqueens}).

After creating all the ASP programs, we turned them into sketches by looking for meaningful opportunities to use 
sketched variables. We introduced sketched variables to replace operators (equalities and inequalities), to replace arithmetic (such as plus and minus)
and to decide whether to use negated literals or not, and to make abstraction of predicates for which another predicate existed with the same signature.

Finally, we had to select the examples in a meaningful way, that is, we selected examples that would be informative
(as a user of SkASP would also do). Positive examples were actually selected more or less random,
negative examples are meant to violate one or more of the properties of the problem. Furthermore, we also 
tried to select examples that carry different information (again as a user of SkASP would do). We selected between 4 and 7 examples for each model. Where relevant in the experiments, we sampled the sketched variables (e.g. \qfive) or the examples (e.g. \qthree)


\textbf{Experimental questions} \sergey{ are designed to evaluate how usable is SkASP in practice. 
  Users want in general to provide only a few examples (\qone-\qthree), to obtain a small number of solutions (ideally only one) (\qone-\qtwo), the examples should be small (\qfour), the solutions should be correct (all), want to know whether and when to use preferences (\qtwo), and how robust the technique is to changes in the encoding (\qfive) as it is well known in ASP that small changes in the encoding can have large effects. Finally, they are typically interested in how the learned programs change as the sketches become more complex (\qthree).
With this in mind,} we have designed and investigated the following experimental questions:
\begin{itemize}
\item \qone: What is the relationship between the number of examples and the number of solutions? How many examples does it take to converge?
\item \qtwo: Are preferences useful?
\item \qthree: What is the effect of the number of sketched variables on the convergence \sergey{and correctness of the learned programs}?
\item \sergey{\qfour: Do models learned on examples with small parameter values generalize to models with larger parameter values?}
\item \sergey{\qfive: What is the effect of encoding variations on the number of solutions and their correctness? }
\end{itemize}

\textbf{Implementation details and limitations.} The SkASP engine is written in Python 3.4 and requires pyasp. All examples have been run on a 64-bit Ubuntu 14.04, tested in Clingo 5.2.0. The current implementation does not support certain language constructs such as choice rules or optimization.

We use the \textit{default preferences} in the experiments for the built-in inequality sketch \sketchedeq: namely $=$ and $\neq$ have equal maximal preference. A user can redefine the preferences. Our experiments indicate that for other sketched types (e.g., arithmetic, etc) no default preferences are needed.

\begin{figure*}[!thb]
  \centering
  \begin{subfigure}[t]{0.32\textwidth}
    \includegraphics[width=\scalefigures\textwidth]{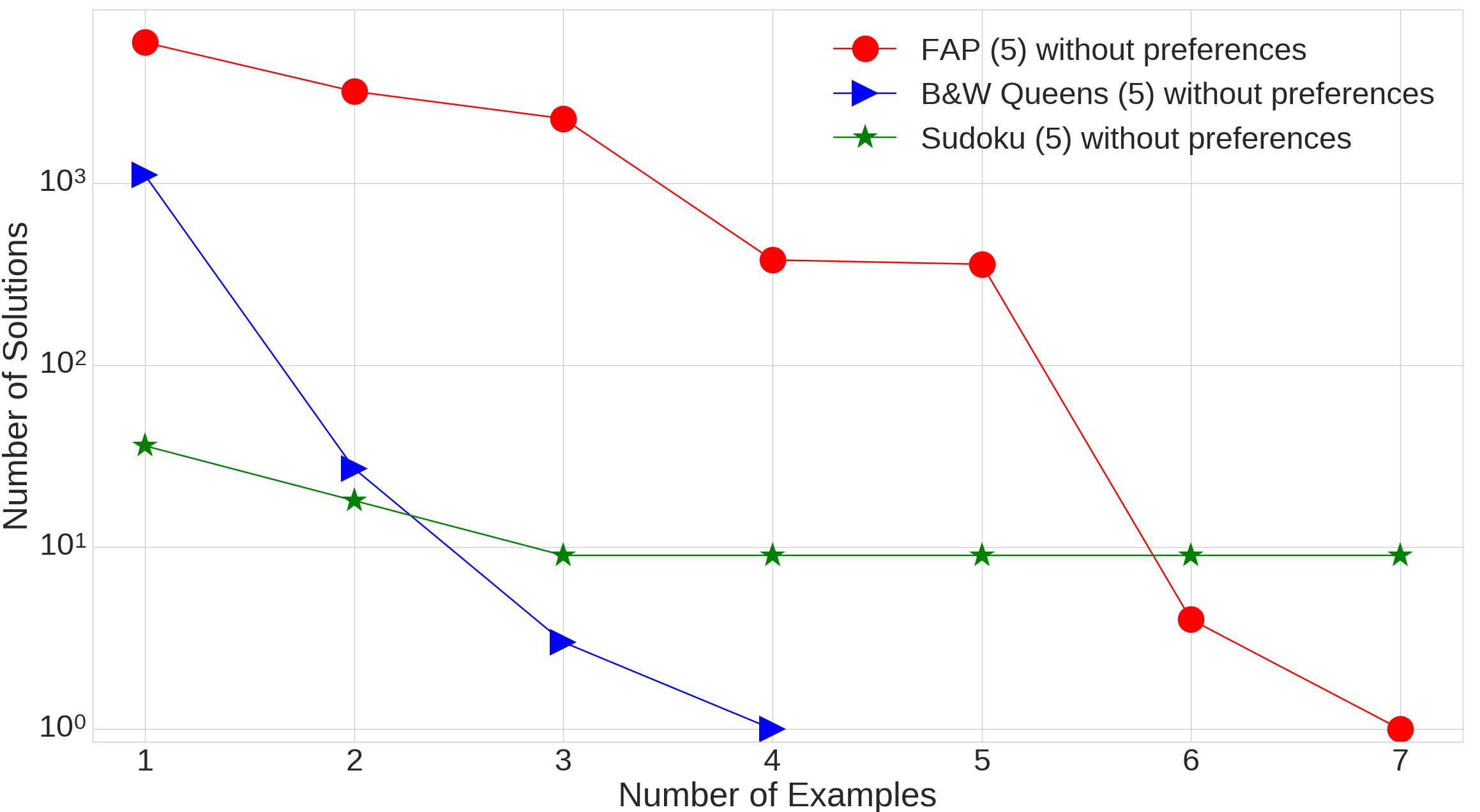}
    \caption{Convergence without preferences (with 5 sketched variables): B\&W Queens (Sketch \ref{lst:queens}) and FAP converge, while Sudoku does not}
    \label{fig:convergence_without_preferences}
  \end{subfigure}
  \hfill
  \begin{subfigure}[t]{0.32\textwidth}
    \includegraphics[width=\scalefigures\textwidth]{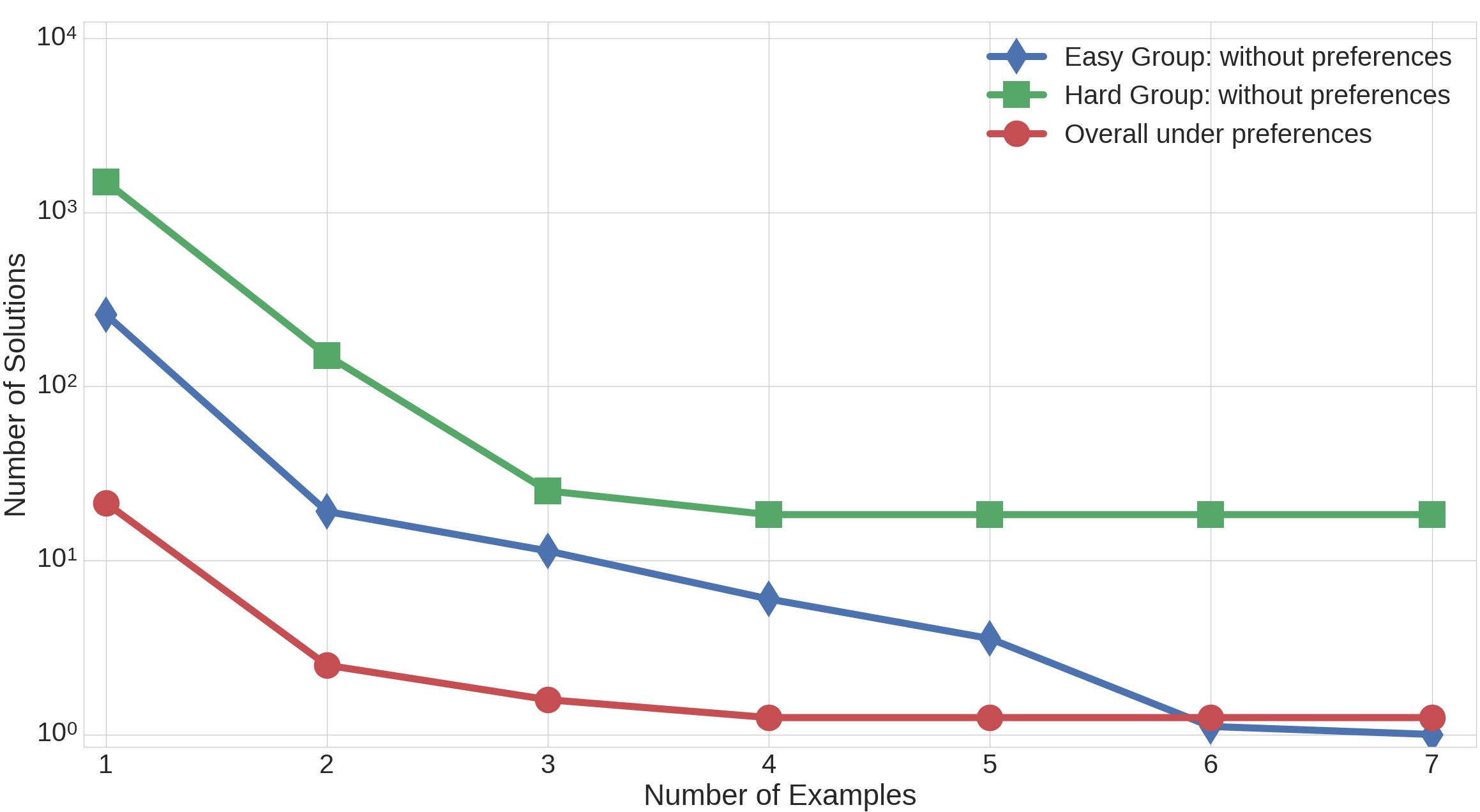}
    \caption{Average number of solutions over the dataset, split into the \textit{easy} group converging without preferences and \textit{hard} not converging }
    \label{fig:preferences_effect}
  \end{subfigure}
  \hfill
  \begin{subfigure}[t]{0.32\textwidth}
    \includegraphics[width=\scalefigures\textwidth]{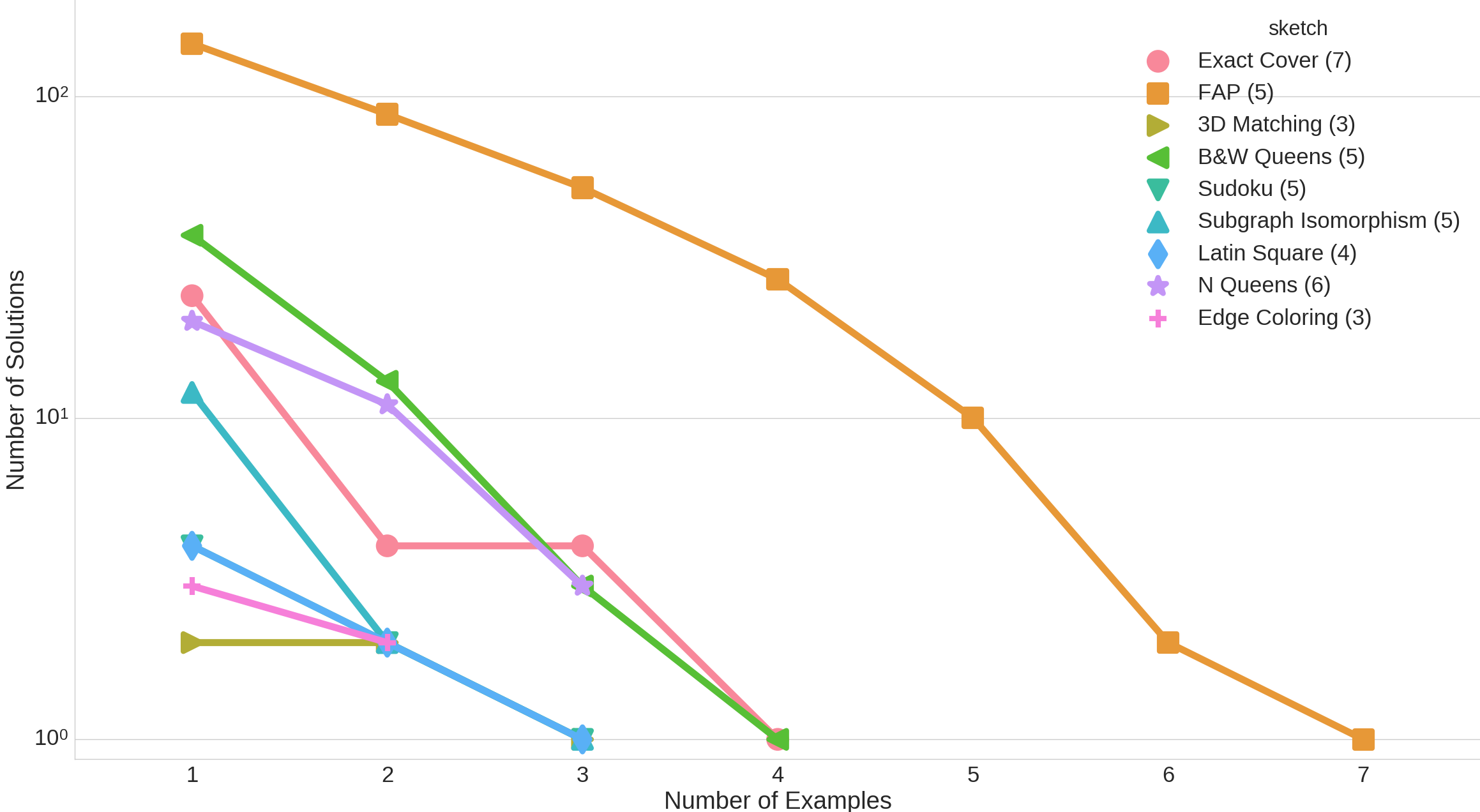}
    \caption{Between 2 and 7 examples are needed to obtain a unique solution (or a small group of equivalent ones) under preferences.}
    \label{fig:number_of_solutions}
  \end{subfigure}

  \begin{subfigure}[t]{0.32\textwidth}
    \includegraphics[width=\scalefigures\textwidth]{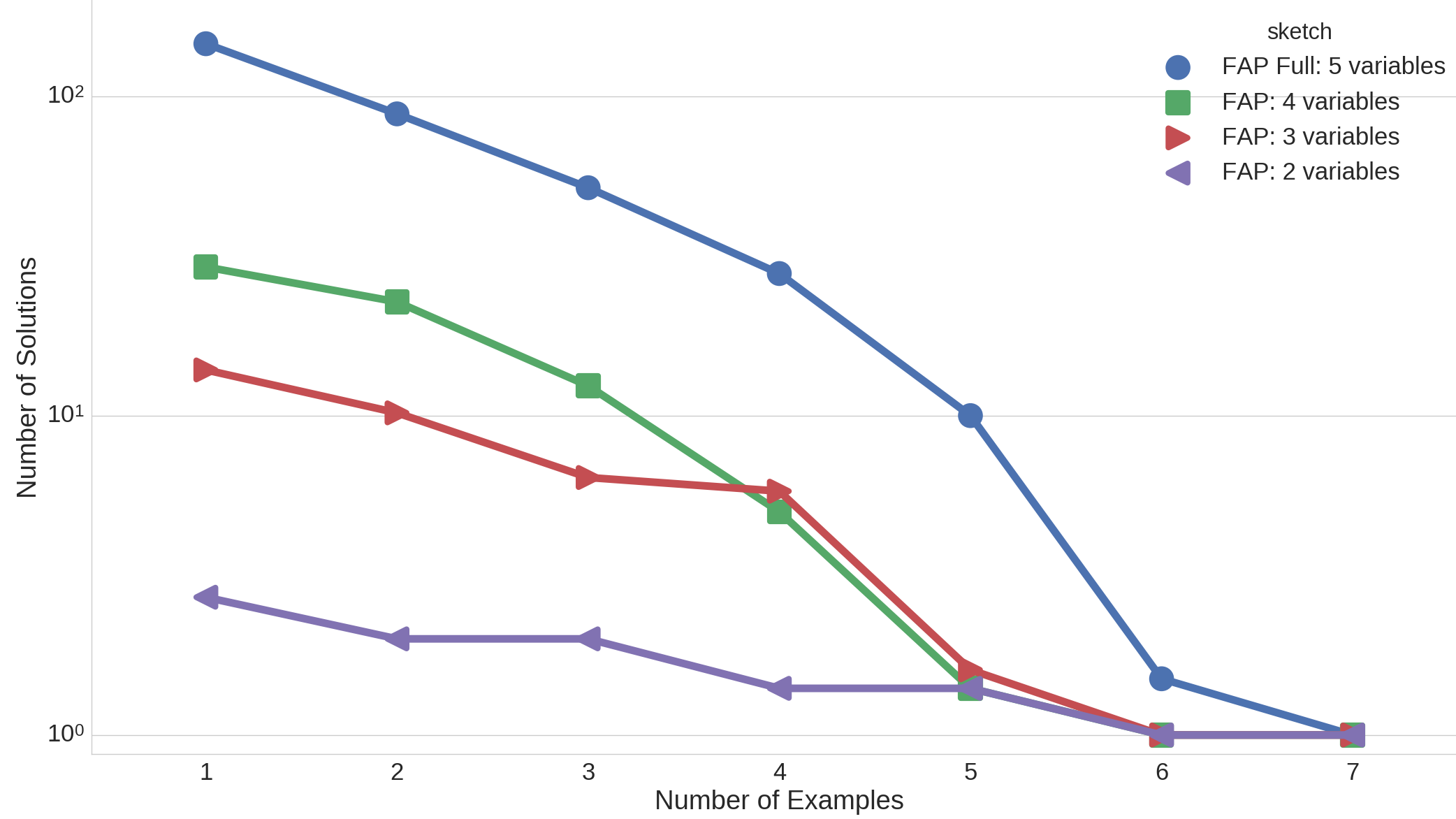}
    \caption{The effect of the number of sketched variables on the solutions with preferences}
    \label{fig:fap_with_preferences}
  \end{subfigure}
  \hfill
  \begin{subfigure}[t]{0.32\textwidth}
    \includegraphics[width=\scalefigures\linewidth]{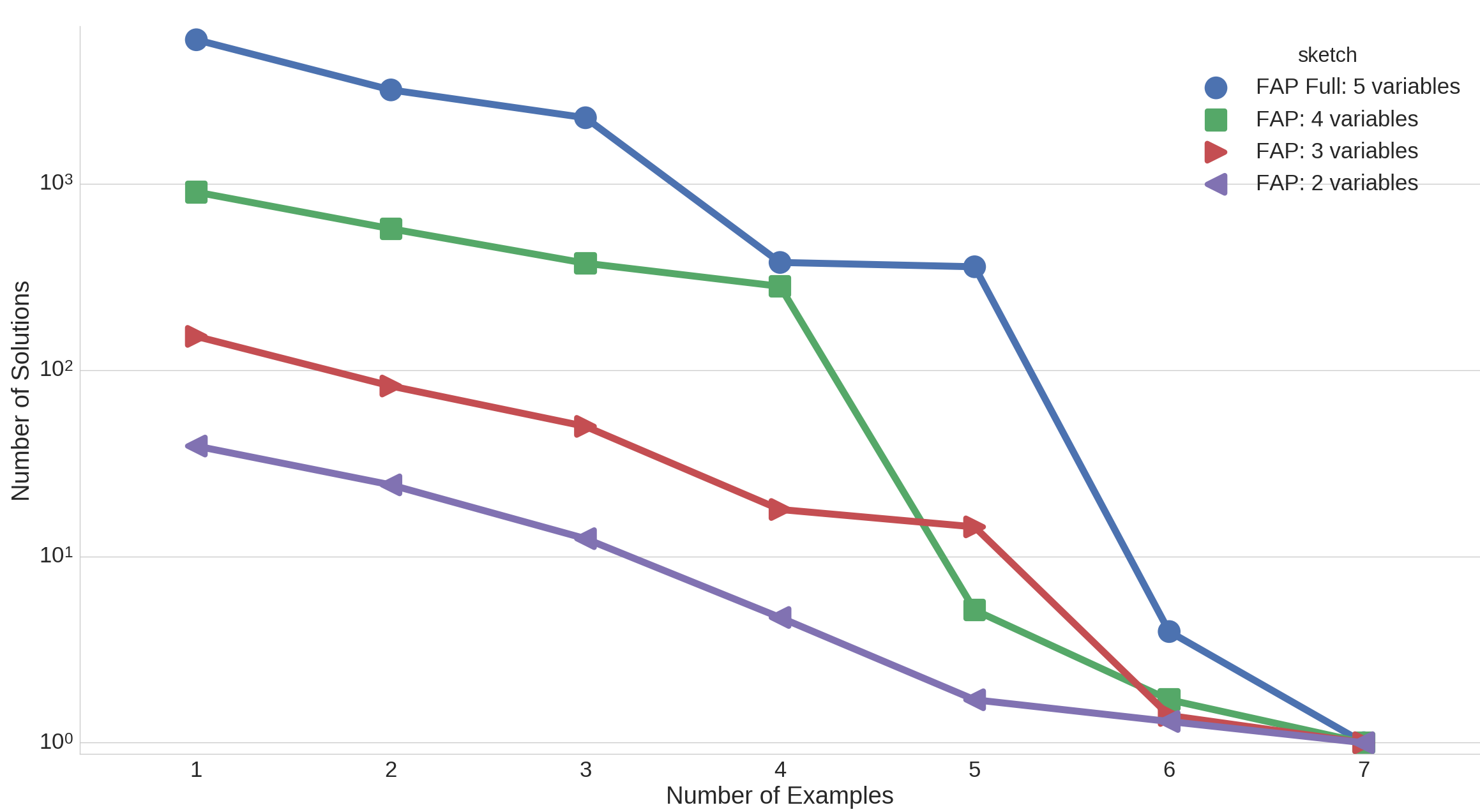}
    \caption{The effect of the number of sketched variables on the solutions without preferences}
    \label{fig:fap_without_preferences}
  \end{subfigure}
  \hfill
  \begin{subfigure}[t]{0.32\textwidth}
    \includegraphics[width=\scalefigures\linewidth]{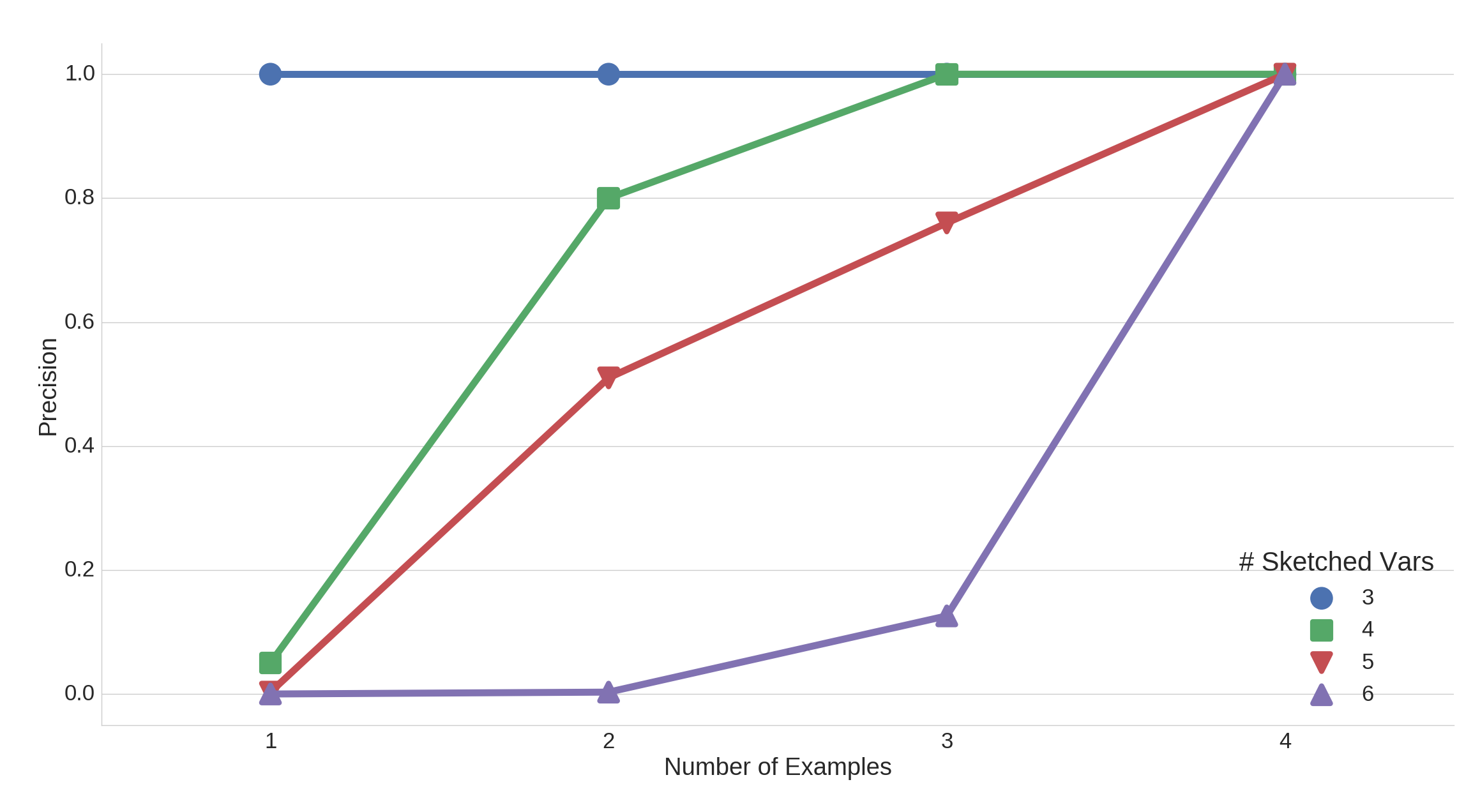}
    \caption{\sergey{Dependency between the number of sketched variables and precision}}
    \label{fig:precision_n_queens}
  \end{subfigure}
  \caption{Addressing experimental questions \qone-\qthree. \qone and \qtwo: convergence without preferences (a), across the dataset (b), for each problem (c). \qthree: the effect of different number of sketched variable on the number of solutions and precision, FAP (d, e); $N$-queens (f). log-scale (a-e)}
\end{figure*}

\begin{figure*}[!thb]
  \centering
  \begin{subfigure}[t]{0.32\textwidth}
    \includegraphics[width=\scalefigures\textwidth]{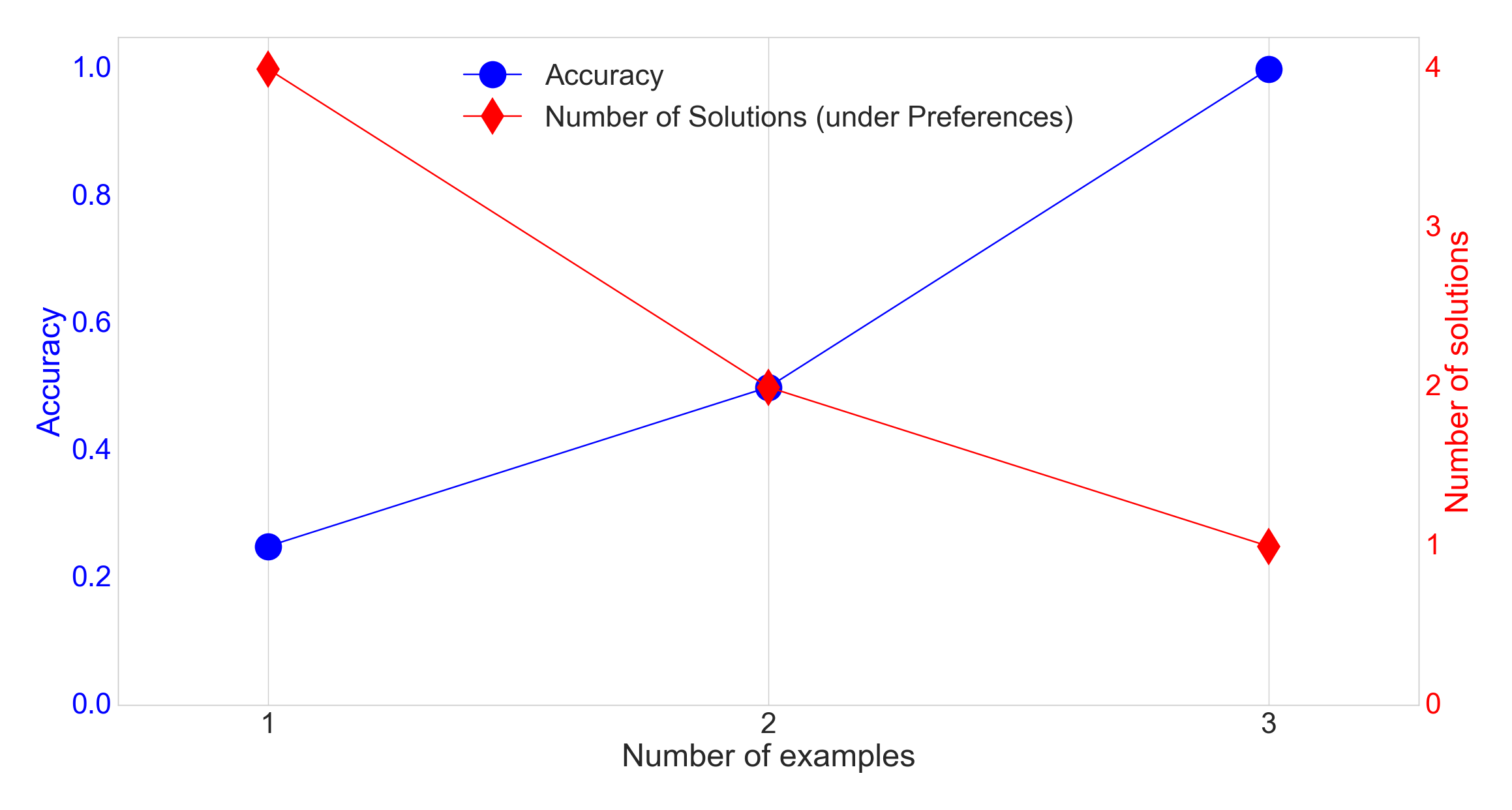}
    \caption{\sergey{Measuring the generalized accuracy and number of solutions. Latin Square $4 \times 4$.}}
    \label{fig:latin_square_generalization}
  \end{subfigure}
  \hfill
  \begin{subfigure}[t]{0.32\textwidth}
    \includegraphics[width=\scalefigures\linewidth]{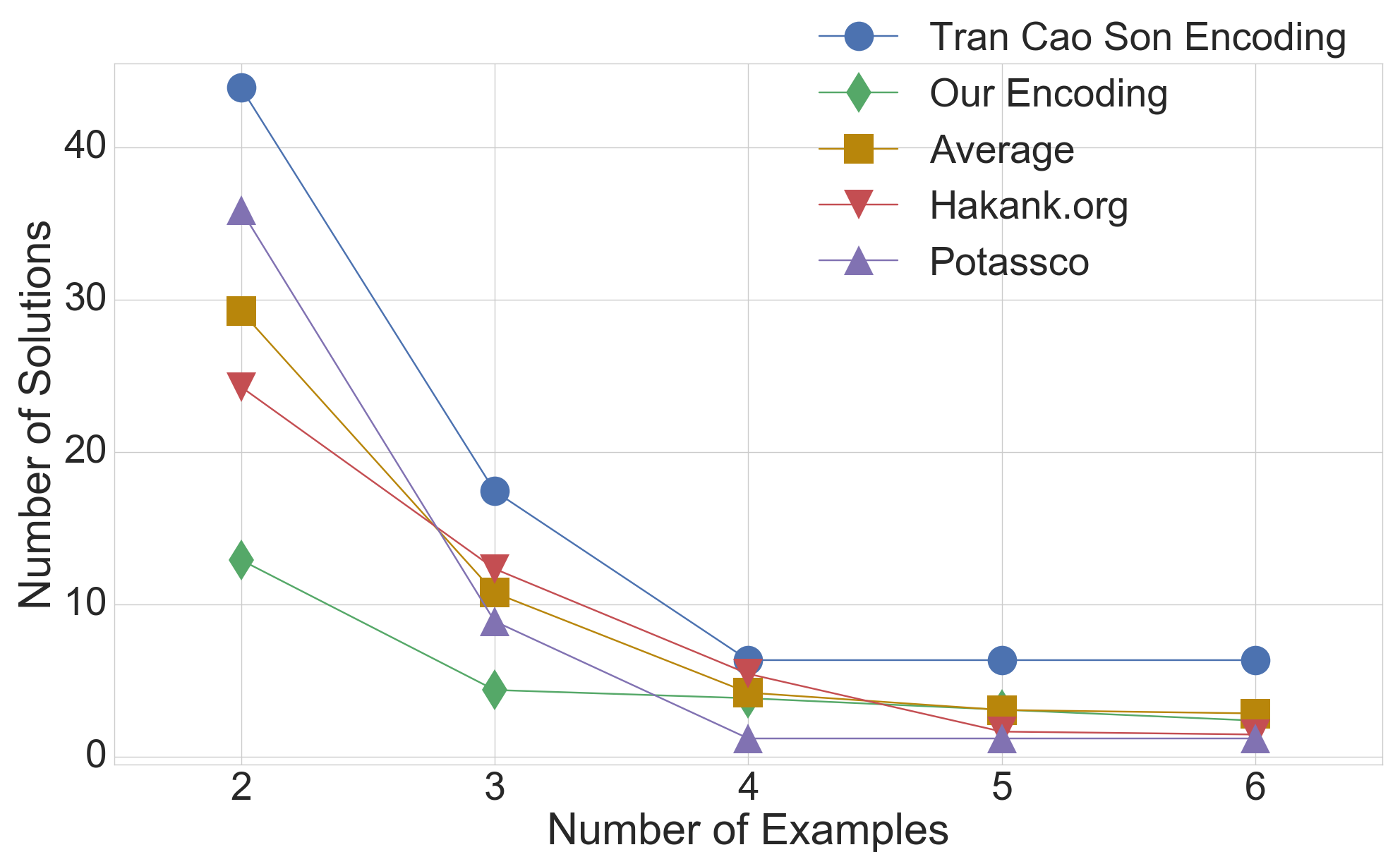}
    \caption{\sergey{The number of solutions (default preferences) for various $N$-queen encodings}}
    \label{fig:systematic_solutions}
  \end{subfigure}
  \hfill
  \begin{subfigure}[t]{0.32\textwidth}
    \includegraphics[width=\scalefigures\textwidth]{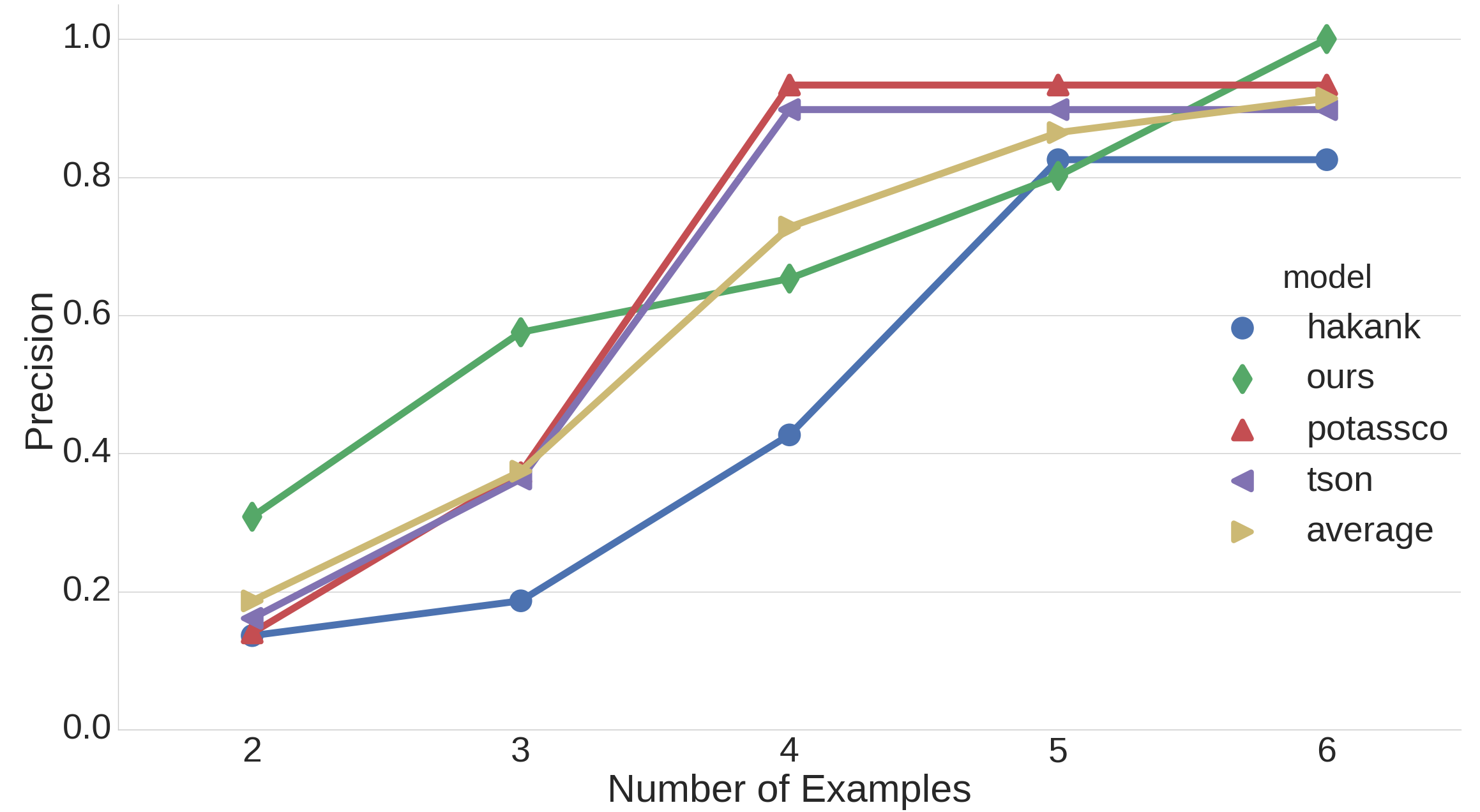}
    \caption{\sergey{Correctness for $N$-queens various encodings (precision)}}
    \label{fig:systematic_precision}
  \end{subfigure}
  \caption{Addressing experimental questions \qfour and \qfive: generalization accuracy on Latin Square (a), the effect of the encoding variations on the number of solutions (b), on precision (c).}
\end{figure*}

We investigate \qone by measuring the impact of the number of examples 
on the number of solutions of the 21 SkASP problems. 
An interesting observation is that even if the user wants to 
solve, say the Latin Square $20\times 20$, she does not need to provide 
examples of size $20\cdot 20= 400$. She can simply provide $3\times 3$ examples 
and our SkASP problem will learn the generic 
Latin Square program (see Figure \ref{lst:latin_square}). 

Figure \ref{fig:convergence_without_preferences} shows how the number of solutions 
for some of our 21 SkASP problems depends on the number of examples. 
In some cases, 7 examples are sufficient 
to converge to a single solution e.g., FAP, B\&W Queens. 

On some other problems, however, after 7 examples there still remain 
many solutions (on average 18 for problems that do not converge). 
Figure \ref{fig:preferences_effect} reports the same information as Figure \ref{fig:convergence_without_preferences} for all 21 problems: the average number of solutions; the average on the 9 that 
converge within 7 examples, referred to as the \textit{easy problems}; and the average on the 12 that still have several 
solutions after 7 examples, referred to as the \textit{hard problems}. 
When SkASP  does not converge to a unique solution, this leaves 
the user with choices, often amongst equivalent ASP programs, which is undesirable.  

For problems that do not converge after a few examples, we propose to 
use preferences, as provided by our SkASP framework. 
We use the default preference described earlier.

We investigate \qtwo by  measuring again the impact of the number of examples 
on the number of solutions. 
In Figure \ref{fig:number_of_solutions}, we observe that all problems 
have converged in less than 7 examples (under default preferences). 
The impact of preferences on the speed 
of convergence is even more visible on the whole set of problems, as reported in 
Figure \ref{fig:preferences_effect}. 
The number of solutions with preferences is smaller, and often much smaller 
than without preferences, whatever the  
number of examples provided. 
With preferences, all our 21 problems are learned with 7 examples.

To analyze the number of solutions in \qthree, we look into the convergence of FAP  
by varying the number of sketched variables. The original 
sketched program of FAP contains 5 sketched variables. 
We vary it from 2 to 5 by turning 3, 2, 1, or 0 sketched variables into 
their actual value (chosen randomly and repeated over multiple runs). As expected, in Figure \ref{fig:fap_with_preferences}, we observe that 
the more there are sketched variables in the SkASP, the slower the number 
of solutions decreases. Furthermore, the number of sketched variables has a greater impact on the convergence without preferences, as we see in Figure \ref{fig:fap_without_preferences}. After 3-5 examples under preferences we have fewer than 10 solutions, while without preferences there are still dozens or hundreds of solutions.

\sergey{To analyze \textit{correctness} in \qthree, we need first to define it. Informally, we mean that the program classifies arbitrary examples correctly, i.e., positive as positive, etc. A typical metric to measure this is \textit{accuracy}. However, there are no well defined arbitrary positive and negative examples for the most problems: what is an arbitrary random example for Feedback Arc Cover? Problems like Sudoku and $N$-queens do have standard examples because they are parameterized with a single parameter, which has a default value. Furthermore, for the standard $8$-queens we know all solutions analytically, i.e., 92 combinations. Another issue is that the negative and positive classes are unbalanced. The usual way to address this issue is to use \textit{precision}, i.e., $\frac{\text{True Positive}}{\text{True+False Positives}}$. (Recall is typically one because the incorrect programs produce way too many solutions that include the correct ones.) In Figure \ref{fig:precision_n_queens}, we see that in all cases we were able to reach the correct solution (here the locations of sketched variables were fixed as specified in the dataset); while increasing the number of sketched variables generally decreases the precision.} 

\sergey{To investigate \qfour, we have used the Latin Square from Listing \ref{lst:latin_square}. We have used examples for Latin Square $3\times 3$, and verified its correctness on Latin Square $4 \times 4$ (which can be checked analytically because  all solutions are known). We have discovered, that there is an inverse dependency between number of solutions and accuracy, see Figure \ref{fig:latin_square_generalization}. This happens because there are typically very few useful or ``intended'' programs while there are lot of incorrect ones.}

To investigate \qfive, we have focused on the $N$-queens problem and collected several encodings from multiple sources: Potascco, Hakank.org, an ASP course by Tran Cao Son\footnote{\url{www.hakank.org/answer_set_programming/nqueens.lp}\\\url{www.cs.uni-potsdam.de/~torsten/Lehre/ASP/Folien/asp-handout.pdf}\\\url{www.cs.nmsu.edu/~tson/tutorials/asp-tutorial.pdf}} and our encoding. 
Whereas all the encodings model the same problem they show significant variations 
in  expressing constraints. To reduce the bias in how the sketched variables are introduced and systematically measure the parameters, we pick sketched variables randomly (inequalities and arithmetic) and use the same examples from our dataset (randomly picking the correct amount) for all models.

In Figure \ref{fig:systematic_solutions}, while there is a certain variation in the number of solutions, they demonstrate similar behavior. \sergey{For each encoding we have introduced $5$ sketched variables and measured the number of solutions and precision. In Figure \ref{fig:systematic_precision} we see that there is indeed a slight variation in precision, with 3 out of 4 clearly reaching above 90\% precision, one reaching 100\% and one getting 82\%. Thus, despite variations in encoding, they generaly behave similarly on the key metrics. }  The results have been averaged over $100$ runs.

Overall, we observe that only few examples are needed to converge to a unique or a small group of equivalent solutions. An example where such equivalent solutions are found is the \text{edge coloring problem}, where two equivalent (for undirected graphs) constraints are found:
\begin{equation*}
  \begin{aligned}
    &\leftarrow \textit{color}(X,Y_1,C), \textit{color}(X,Y_2,C),~Y_1 \neq Y_2.\\
    &\leftarrow \textit{color}(X_1,Y,C), \textit{color}(X_2,Y,C),~X_1 \neq X_2.
  \end{aligned}
\end{equation*}
For this problem these two constraints are equivalent and cannot be differentiated by any valid example. 

An interesting observation we made on these experiments is that the 
hardness (e.g., in terms of runtime) of searching for a solution of 
a problem is not directly connected to the hardness of learning the 
constraints of this problem. This can be explained by the incomparability of the search spaces. 
SkASP searches through the search space of sketched variables, 
which is usually much smaller than  the search space of the set of decision variables 
of the problem to learn.

\section{Related Work}\label{sec:related_work}
The problem of sketched ASP is related to a number of topics.
First, the idea of sketching originates from the area of programming languages, where
it relates to so called self-completing programs \cite{sketching_phd_thesis}, typically in C \cite{sketching_original} and in Java \cite{jsketch}, where an imperative program has a question mark instead of a constant and a programmer provides a number of examples to find the right substitution for it. While sketching has been used in imperative programming languages, it has -- to the best of the authors' knowledge -- never been applied to ASP and constraint programming.  What is also new is that the sketched ASP is solved using a standard ASP solver, i.e., ASP itself.

Second, there is a connection to the field of inductive (logic) programming (ILP)  \cite{ilp_book,ilp_original,gulwani2015inductive}.
An example is meta-interpretive learning \cite{MuggletonMLJ14,MuggletonMLJ15} where 
a Prolog program is learned based on a set of higher-order rules, which  act as a kind of template
that can be used to complete the program. 
However, meta-interpretive learning differs from SkASP in that it induces full programs 
and pursues as other ILP methods a search- and trace-based approach guided by generality, 
whereas SkASP using a constraint solver (i.e., ASP itself) directly.  Furthermore, the target is different in that ASPs are learned, which include constraints. 
SkASP relates to meta-interpretation in ASP \cite{asp_meta} in rule and decision materialization. The purpose is, however, different: they aim at synthesizing a program of higher complexity ($\Sigma_2^P$) given programs of lower complexity (\textit{NP} and \textit{Co-NP}).

There are also interesting works in the intersection of ILP, program synthesis and ASP \cite{inductive_asp,inductive_asp2,XHAIL}. The ILASP system \cite{ILASP_system} learns an ASP program from  examples, and a set of modes, while minimizing a metric, typically the number of atoms. This program, learned completely from scratch, is not necessarily the best program from the user's point of view and may limit the possibility to localize the uncertainty based on the user's knowledge of the problem.  Indeed, if all sketched predicates are added in the modes with corresponding background knowledge, then the set of hypotheses of sketched ASP is a subset of ILASP. However, if we specify a sketched constraint \texttt{:- p(X),q(Y),X?=Y} with the negative example \texttt{\{p(1),q(2)\}} as modes for ILASP \cite{ILASP_system}, it would learn a program like \texttt{:- p(X)} (minimal program), but that is clearly not the program intended by the sketch. Furthermore, we compute all preferred programs instead of a single solution. 

Third, there is also work on constraint learning, where the systems such as CONACQ \cite{original_constraint_learning,besetalAIJ17} and QUACQ \cite{QUACQ} learn a set of propositional constraints,  and ModelSeeker \cite{BeldiceanuS12} learns global constraints governing a particular set of examples.  The subject has also been investigated in ILP setting \cite{lallouet}. 
However, the idea in all these approaches is to learn the complete specification of CSPs from scratch. Our setting is probably more realistic 
from a user perspective as it allows to use the knowledge that the user no doubt possesses about the underlying problem,
and also requires much less examples.  On the other hand, SkASP also makes, as other sketching approaches,
the strong assumption that the intended target program is an instance of the sketched one. This may not always
be true, for instance, when rules are missing in the program.  This is an interesting issue for further research.

Fourth, our approach is related to debugging of ASP \cite{debugging_using_meta_asp,debugging_asp}. 
Unlike SkASP such debuggers can be used to locate bugs, but typically 
do not provide help in fixing them. On the other hand,
once a bug is identified, SkASP could be used to repair it by introducing a sketch and a number of examples\footnote{ {\scriptsize During the experiments, we stumbled upon a peculiar bug. One ASP encoding that we discovered in a public repository worked mostly by pure luck. The following constraint \texttt{:-queen(X1,Y1),queen(X2,Y2),X1<X2,abs(Y1-X1)==abs(Y2-X2).} works because \texttt{abs} is not actually absolute value but an 
uninterpreted function, essentially it checks $X == Y$, and that is indeed the found solution. \sergey{(This kind of bugs would be extremely hard to find using traditional debuggers, since technically the encoding produced correct solutions.)}. 
Also, while working on the aggregate extension use-case, we discovered a subtle bug: the case of a single celebrity was not handled correctly. In both cases, the author has been contacted and models have been updated.}} The approach of \cite{ilp_debugging_asp} is based on classical ILP techniques of generalization and specification and does not provide the freedom to indicate uncertain parts of the program.

\section{Discussion and Conclusions}\label{sec:discussion}
Our contribution is four-fold: we have introduced the problem of sketched ASP; we have provided a rewriting schema for SkASP; we have created a dataset of sketches and we have evaluated our approach empirically demonstrating its efficiency and effectiveness.

User interaction is an interesting future direction, namely to suggest constraints and examples. For the former, if we are not able to reject a negative example, we can construct a constraint that would reject the negative examples and none of the positive examples. As for the examples, if we have two solutions to a problem, we can generate an example discriminating between them and ask user to clarify it, while this might not always be possible, since symmetric assignments might lead to semantically identical programs. In practice, however, this might be an important addition to simplify sketching for end users. Another direction is to incorporate \sergey{non-constant} preference handling into the model using the extensions of ASP for preference handling, such as \textit{asprin} \cite{asprin}.

\bibliography{references}
\bibliographystyle{splncs}
\end{document}